\documentclass[11pt]{article}

\usepackage[margin=1in]{geometry}
\usepackage{graphicx}
\usepackage[numbers]{natbib}
\usepackage[utf8]{inputenc}
\usepackage[T1]{fontenc}
\usepackage[english]{babel}
\usepackage{url}
\usepackage{amsmath, amssymb, amsthm}
\usepackage{lmodern}
\usepackage{booktabs}
\usepackage{multirow}
\usepackage{siunitx}

\newtheorem{theorem}{Theorem}
\newtheorem{proposition}[theorem]{Proposition}

\title{Prediction-powered estimators for finite population statistics in highly imbalanced textual data: Public hate crime estimation}

\date{}

\author{
  Hannes Waldetoft\thanks{Department of Statistics, Uppsala University.}
  \and
  Jakob Torgander\footnotemark[1]
  \and
  Måns Magnusson\footnotemark[1]  
}

\begin{document}

\maketitle

\begin{abstract}
\noindent Estimating population parameters in finite populations of text documents can be challenging when obtaining the labels for the target variable requires manual annotation. To address this problem, we combine predictions from a transformer encoder neural network with well-established survey sampling estimators using the model predictions as an auxiliary variable. The applicability is demonstrated in Swedish hate crime statistics based on Swedish police reports. Estimates of the yearly number of hate crimes and the police's under-reporting are derived using the Hansen-Hurwitz estimator, difference estimation, and stratified random sampling estimation. We conclude that if labeled training data is available, the proposed method can provide very efficient estimates with reduced time spent on manual annotation.
\end{abstract}

\section{Introduction}

Estimation in finite populations has a long tradition in statistics \cite{cochran1977sampling,hansen1943theory,horvitz1952,sarndal2003model}. In many applications, totals and proportions of a known finite population are of interest, e.g., when sample surveys are used in research \cite{magnussen2011horvitz,rudolph2014estimating,boo2022high}. It is also the common interest in the production of official statistics.

Official statistics should provide reliable, lasting and high-quality statistical estimates of importance to society \cite{tille2022some}. Hence, it provides valuable information for policy making, research, and the general public. Multiple quality dimensions serve as guiding principles in official statistics, including relevance, accuracy, timeliness, accessibility, and coherence \cite{StatisticsCanada2017}. Of these dimensions, accuracy plays a crucial role and is in many countries stipulated in law \cite{swedish_law,EUReg223_2009,NAP27934}.

In official statistics, and more recently in research, it is common for the population of interest to consist of textual data or collections of documents \cite{bonikowski2022,do_ollion2024augmented}. Statistical estimates of interest can be totals or proportions of the document collection with specific properties, e.g., the total number of hate crimes during a year, particular symptoms in medical journals, the number of industries or businesses belonging to a specific category, or identifying populism, nationalism and authoritarianism in political speeches \cite{luo2021deep,yang2025emerging,bonikowski2022}. In these settings, producing accurate statistics can be challenging if there are no existing labels of interest or if the labels are erroneous, requiring manual expert judgment \cite{rilkoff2024innovations,anmalda_brott}. Extracting the desired information from a text document can require specialized knowledge of the relevant domain, and a large volume of documents can make manual processing very time-consuming, especially if the documents of interest are a tiny proportion of the general population \cite{settles2009active, LundgrenLejonstad2023,bonikowski2022}. 

An increasingly popular approach is to reduce the annotation burden by employing machine learning methods to predict the labels of interest, sometimes referred to as prediction-powered inference \cite{angelopoulos2023prediction} or augmented annotation \cite{do_ollion2024augmented}. However, the naive use of prediction models introduces bias \cite{egami}, making such an approach less relevant in official statistics or other applications where accurate statistical inference is required.

\subsection{Text classification and finite-population estimation}

Text classification, which involves categorising documents or text segments, has been a long-standing research topic \cite{eisenstein2018natural,wankhade2022survey,piryani2017analytical,kowsari2019text}. In 2017, a significant performance leap in text classification was achieved with the introduction of the transformer architecture \cite{vaswani2017attention}, and especially with pre-trained encoder-based transformer models such as the BERT and roBERTa models \cite{devlin2018bert,liu2019roberta}. These models, pre-trained on large textual corpora, are often state-of-the-art when used in downstream tasks, such as text classification \cite{beltagy2019scibert,sun2019fine,gasparetto2022survey}. These encoder-based models, fine-tuned on human annotations, have been shown to often outperform Large Language Models (LLMs) for classification \cite{ollion2023chatgpt,fields2024survey,kristensen2025chatbots}, and have a lower computational cost \cite{samsi2023words}.

The use of predictive models, such as text classifiers, is increasingly prevalent in finite-population settings \cite{beck2018machine, tille2022some}, often with the primary purpose of identifying or classifying units, as in household or business statistics. Random forests, neural networks, and support vector machines are commonly used, e.g. in the use of shrinkage, calibration, and treatment of nonresponse \cite{beck2018machine,tille2022some}. Similarly, LASSO, or $L_1$ regularization, has been proposed as a solution when a large number of auxiliary variables are available, e.g., from remote sensing data \cite{mcconville2017lasso}. However, these examples do not use the predictive models to improve the estimation of finite population parameters.

Recently, \textit{prediction-powered inference} (PPI) has been introduced \cite{angelopoulos2023prediction} as a method that leverages predictions from black-box machine learning models to enhance the estimation of target quantities, including means or regression coefficients, in conventional statistical settings. As an example, \cite{egami} uses large language models to boost statistical inference in regression models by combining ground truth labels and predicted labels. At the same time, \cite{egami} maintains asymptotically unbiased parameter estimates of model coefficients. However, \cite{egami} does not address relevant estimands such as population totals in finite-population settings, nor the fact that encoder-based transformers often outperform LLMs in text classification settings \citep{kristensen2025chatbots}.

\subsection{Swedish hate crime statistics}

Hate crimes are offences committed with a bias motivation, driven by fear, prejudice, or hatred toward specific demographic groups. In Sweden, for most criminal offences, a hate motive is an aggravating factor during sentencing \cite[Ch 29, § 2(7)]{SCC}. The definition of a hate-motivated crime, as defined by the Swedish National Council for Crime Prevention (SNCCP), is: \textit{crimes motivated by the perpetrator's prejudices or beliefs about race, colour, nationality, ethnic background, creed, sexual orientation or transgender identity or expression.}, \citep[our translation]{LundgrenLejonstad2023}. The focus is on the motive behind the offence rather than the specific type of offence committed.

Hate crime statistics commonly intend to estimate the prevalence of hate crime motives in the population of police reports \cite{TODO}, and are often produced to contribute knowledge for research, follow international obligations to report hate crime statistics, examine the police's reporting of hate crimes, and assist the judicial system with a basis for decisions aimed at preventing hate crimes~\cite{LundgrenLejonstad2023}. 

From 2007 to 2019, hate crime statistics in Sweden were produced by extracting all police reports containing one or more elements from a set of keywords strongly indicative of hate crimes. Of these selected reports, 50\% were randomly sampled and manually annotated by experts at SNCCP. As of 2020, the manual annotation subset comprises police reports manually flagged by the police as having a hate crime motive, and all these flagged reports are then manually annotated by experts at SNCCP. 

Since 2016, SNCCP's hate crime statistics have been derived every two years. The reason for this is the demanding and costly manual annotation procedure. Annotating police reports to determine whether a hate crime motive is present is time-consuming, especially since complex cases are often discussed among several experts, requiring a substantial amount of time \cite{LundgrenLejonstad2023}. In addition, the low proportion of police reports containing hate crime motives, only about 0.4\%, makes effective sampling from the population difficult.

The manual flagging by the police is performed directly when the police officer creates the report by responding to a prompt in the reporting software, which asks whether a hate crime motive is present. They have made this annotation since 2008, but it has only been used for hate crime statistics since 2020. Before 2020, the accuracy of police reporting was deemed insufficient \cite{snccp2018}. Currently, approximately half of the flagged reports are hate crimes \cite{LundgrenLejonstad2023}. In 2020 and 2022, of the approximately 1.5 million reports filed yearly, the police marked 6300 and 4800 reports as hate crimes, respectively. Out of these, SNCCP experts concluded that 54\% and 56\%  respectively were hate crimes. The most common types of hate crime motivations are based on xenophobia and race, at 53\% in 2022, followed by hate crimes against religious groups and LGBTI-related hate crimes, at 16\% and 12\% in 2022, respectively. Figure \ref{fig:number_hatecrime} displays the yearly numbers of confirmed hate crimes. 

\begin{figure}
    \centering
    \includegraphics[width=1\textwidth]{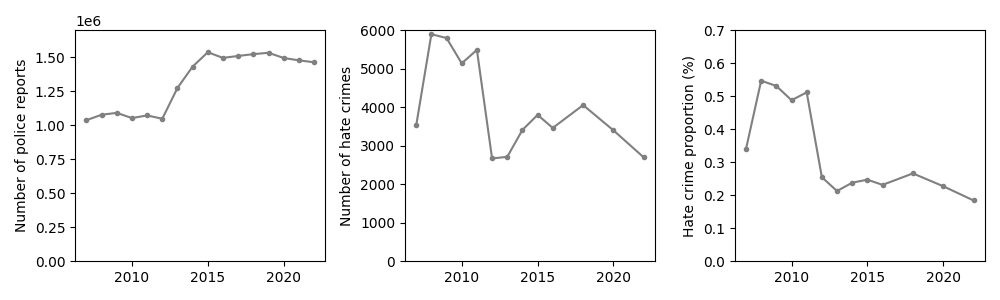}
    \caption{Number of police reports filed for the period 2007-2022, the number of confirmed hate crimes for this period, and their proportion of the total. Note that SNCCP did not produce hate crime statistics in 2017, 2019, and 2021.}
    \label{fig:number_hatecrime}
\end{figure}

Overall, the small proportion of hate crimes of all reported crimes also makes it difficult to sample effectively for specific hate crime types, e.g., to estimate the proportion of antisemitic hate crimes. Despite the challenges in sampling and estimating such small proportions, these crimes are of specific interest and are a key part of the statistics produced \cite{LundgrenLejonstad2023}.

Finally, the current approach of the SNCCP, which relies on flagged police cases, effectively solves the problem of selecting reports to annotate. However, outside of this subset in the police reports that were not flagged, it is not known how many hate crimes are missed. Hence, unbiased estimates of the yearly number of hate crime police reports cannot be obtained when relying solely on police annotations.

\subsection{Our contribution}

This work combines transformer-based neural networks with finite-population sampling estimators to obtain unbiased estimates of finite population parameters for extensive textual document collections. We present the proposed approach for Swedish hate crime statistics, a highly unbalanced textual data setting where standard sampling approaches without utilising auxiliary information are practically infeasible.

\newpage
\section{Prediction-powered finite-population estimation} 
\label{sec:estimators}

We propose a straightforward approach to estimating highly unbalanced proportions using prediction-powered sampling inference. The primary interest is the total (or proportion) of a binary variable in a finite population, i.e. $y\in\{0,1\}$, $t=\sum_{U}y_i$, where $U = \{1,...,N\}$ is a set of indices, representing the $N$ units of the population. Note that here, $y_i$ is fixed, but unobserved. The core idea of the prediction-powered finite population estimator has three steps:
\begin{enumerate}
    \item Train a classifier, in our case a transformer encoder neural network, denoted $f(x_i; \theta)$, where $x_i$ represents the texts of a police report $i$ and $\theta$ denotes the model parameters or network weights, on a binary classification task predicting the presence or absence of the label of interest $y_i$. Note that this classifier has been trained on different data from the finite population for which estimates are to be made.
    \item Use the model predictions for observations $x_i$, where the label $y_i$ is not observed, as an auxiliary variable used to sample from the finite population. 
    The auxiliary variable can be either the predicted probability $\hat{p}_i=f(y_i=1|x_i;\theta)$ or the predicted class \[
    \hat{y}_i =
    \begin{cases}
    1, & \text{if } \hat{p}_i \geq \tau \\
    0, & \text{if } \hat{p}_i < \tau
    \end{cases}
\]
where $\tau$ is the classification threshold. 
\item Take a random sample of size $n$ from the population $U$ and use sampling-based estimators to estimate the finite population estimands of interest. For example, we are interested in the total number of hate crimes in a finite population of police reports in the year 2022.
\end{enumerate}

We employ this approach to estimate Swedish hate crime statistics using three different standard estimators: probability-proportional-to-size (PPS) sampling with the Hansen-Hurwitz estimator, stratified random sampling, and stratified difference estimation.

\subsection{The Hansen-Hurwitz prediction-powered (H2P2) estimator}
We use the Hansen-Hurwitz (HH) estimator \cite{hansen1943theory,cochran1977sampling} to sample observations with replacement and sample with inclusion probabilities proportional to the auxiliary variable $\hat{p}_i$ produced by the prediction model, i.e., $\hat{p}_i=f(x_i|\theta)$. Our focus is text classification transformer prediction models, but in practice, any predictive model can be used. The total $t$ is estimated using the standard HH estimator $\hat{t}_{HH}$ which is given as follows,
\begin{equation*} 
\label{eq:hansen_hurwitz}
\begin{aligned}
    \pi_i &= \frac{\hat{p_i}}{\sum_{j \in \mathcal{S}} \hat{p_j}}\,, \quad &
    \hat{t}_{HH} &= \frac{1}{n}\sum_{i \in\mathcal{S}}\frac{y_i}{\pi_i}\,,
\end{aligned}
\end{equation*}
where $\pi_i$ is the probability of including the $i$-th unit of $U$ in the sample $\mathcal{S}=\{s(1),...,s(n)\}$, with $s(j) = k$ encoding that the $k$-th element of $U$ is drawn at sampling step $j$. For the variance, we have
\begin{equation*}
V(\hat{t}_{HH)} = \frac{1}{n} \sum_{i \in\mathcal{S}} \pi_i\left(\frac{y_i}{\pi_i} - t \right)^2
\end{equation*}
which is unbiasedly estimated by
\begin{equation} \label{eq:hh_variance}
    \hat{V}(\hat{t}_{HH}) = \frac{1}{n}\frac{\sum_{i \in \mathcal{S}}\left(\frac{y_i}{\pi_i}-\hat{t}_{HH}\right)^2}{n-1}
\end{equation}
The HH estimator is unbiased \citep[see][]{cochran1977sampling}. Additionally, if we select an auxiliary variable that is highly correlated with the target $y$, we can reduce the variance of the estimator. This means that we can improve the estimator's efficiency by enhancing the classifier's performance on the finite population of interest. We formalise this as Proposition \ref{prop_first}.

\begin{proposition} 
\label{prop_first}
Let $\ell_i(\hat p_i, y_i)=-y_ilog(\hat{p}_i)-(1-y_i)log(1-\hat{p}_i)$ denote the (pointwise) binary cross-entropy loss for the $i$-th point of our data set. Then, assuming that  the classifier improves and the total loss $\ell= \sum_i \ell_i$ reaches its theoretical minimum of zero, 
$$
V(\hat{t}_{HH}) \rightarrow 0.
$$
\end{proposition}

\begin{proof}
Let $S^+ = \{i \in \mathcal{S};\,y_i=1\} $ and $S^-=\{i\in \mathcal{S};\, y_i =0\}$ be a partition of $\mathcal{S}$, corresponding to the positive and negative labels of our \textit{resampled} data set. The variance function in Equation \eqref{eq:hh_variance} can then be decomposed, up to the normalising constant $\frac{1}{n(n-1)}$, as follows,
\begin{align*}
    V(\hat{t}) &\propto \sum_{i\in S^+}\pi_i(\frac{y_i}{\pi_i}-t)^2 + \sum_{i\in S^-}\pi_i(\frac{y_i}{\pi_i}-t)^2 \\
    &= \sum_{i\in S^+}\pi_i(\frac{1}{\pi_i^2} - \frac{2t}{\pi_i} +t^2) + \sum_{i\in S^-}\pi_it^2 \\
    &= \sum_{i\in S^+}\frac{1}{\pi_i} - \sum_{i\in S^+}2t +\sum_{i\in S^+} \pi_it^2 + \sum_{i\in S^-}\pi_it^2 \\
    &= \sum_{i\in S^+}\frac{1}{\pi_i} - t^2,
\end{align*}
where we in the last equality use that $|S^+| = t$ and that $\mathcal{S}$ is the disjoint union of $S^+$ and, $S^-$. Now,  the cross entropy loss can for any $k$ be rewritten as follows,
\begin{align*}
    e^{\ell_k} =\hat{p}_k^{-y_k}(1-\hat{p}_k)^{- (1 - y_k)} = \begin{cases}
        \hat{p}_k&k \in S^+ \\
        1-\hat{p}_k & k \in S^-
    \end{cases} .
\end{align*}
Since by construction $\hat p_k \in (0,1)$, it holds  $l_k > 0$ for all $k$. Hence, it is easy to see that the total loss $l$ tend to zero if and only if $l_k \to 0$ for all $k$. This yields in turn,
\begin{align*}
    \lim_{\ell \to 0}  \hat{p}_k = \begin{cases}
        1, & k \in S+ \\
        0 & k \in S^-
    \end{cases}.
\end{align*} and hence
\begin{equation*}
    \lim_{\ell \to 0} \pi_k =    \lim_{\ell \to 0} \frac{\hat{p_k}} {\sum_{j \in \mathcal{S}} \hat{p_j}} =
        \frac{1}{\sum_{i \in S^+}1 + \sum_{i \in S^-}0} = \frac{1}{t}.  
\end{equation*}
 Since this clearly holds for any $k \in S^+$ we finally receive 
\begin{equation*}
    \lim_{\ell \to 0}V(\hat{t}) = \sum_{i\in S^+}t - t^2  = 0,
\end{equation*}

which was to be shown.
\end{proof}

\subsection{Stratification-by-prediction (SbP) estimation}

There are clear benefits in using the H2P2 estimator for unbiased estimation of the total of interest. However, in many settings, there is an interest in using population stratification instead. At the SNCCP, the primary interest is not necessarily only the total number of hate crimes, but different domain estimates, such as the number of antisemitic hate crimes. Hence, an alternative is to use the predicted category as a stratification variable \cite{sarndal2003model}. We then get two strata. For convenience, we refer to the stratum consisting of all police reports predicted by the model to be hate crimes ($\hat{y}_i=1$) as the ``one-stratum'', and the stratum of all predicted non-hate crimes ($\hat{y}_i=0$) as the ``zero-stratum''. We can then choose different sampling strategies in the strata defined by the classifier. For example, we can focus the sampling effort on the one-strata where, if the classifier is good, the majority of hate crimes could be found.

\paragraph{Difference estimation in the zero-strata}
Suppose we use the classifier for stratified sampling and estimation. In that case, the zero-stratum will contain fewer and fewer positive observations ($y_i=1$) as the classifier performance in cross-entropy loss improves in the population of interest. 

We can use either the H2P2 estimator in the zero-stratum or the difference estimator \cite{cochran1977sampling}, which utilises the auxiliary variable in estimation, but only with a simple random sample. This leverages the auxiliary information from the classifier in addition to the stratification. The estimate of the total and its variance can then be combined based on the estimates in the two strata.

\section{Transformer-powered hate crime estimation}

Our use case for prediction-powered finite-population estimation is the Swedish hate crime statistics, which are based on Swedish Police reports and expert-annotated police reports.

\subsection{Swedish police report data} \label{sec:police_reports}

Between 2007 and 2022, a total of 21.6 million police reports were filed in Sweden. Among these, \num{485000} reports did not contain any text and are excluded from our study, as is done in the current production of hate crime statistics. Of the remaining reports, SNCCP experts have manually classified a total of \num{52046} reports as having a hate crime motive and \num{61741} reports as not having a hate crime motive. The remaining reports are not annotated (21.0 million). These remaining reports were not included in the SNCCP hate crime sample selection based on keyword search (from 2007 to 2019) or were not flagged as hate crimes by the police (in 2020 and 2022). The cases annotated as hate crimes have additional information regarding the type of hate crime (e.g., antisemitic or afrophobic hate crime).

When a police report is created, the suspected crime is briefly described in text, and the police officer indicates whether a hate crime motive is suspected in the reporting software. The descriptive text in a police report is often just a few sentences long, describing the crime and sometimes providing additional brief information, such as whether the action was captured on a surveillance camera or if the victim requests financial compensation. Two fictitious examples of the text section of police reports, where the second one has a hate crime motive while the first one does not, are:

"\textit{CRIME Simon works as a social security worker. During the last week, he has received several calls and text messages from phone number XXX with threats related to an ongoing case, saying that he will get in serious trouble if the case is not dropped. Simon fears for his safety because of this.}", 
and 
"\textit{CRIME Alicia reports an incident occurring in a social media thread. She posted a wedding picture of her and her wife. In the comments section, the account [Surname] [Lastname] wrote that it was "disgusting" and that "marriage should not be allowed for people like you". Printscreens are saved and can be supplied}". 

The police reports are, in general, written in the Swedish language. The average number of words per report in the available data is 103.

A challenge is that among flagged reports in 2020 and 2022, 289 and 233 cases, respectively, were labelled ``Probably not'' by experts at SNCCP due to them having hate crime-like characteristics while not fully meeting the criteria for hate crime motives. We excluded these cases from training and kept them as a separate class for testing and evaluation due to their known difficulties, also for SNCCP experts. Hence, the data consists of four types of data: (1) the true hate crimes, obtained from the SNCCP annotations; (2) the true non-hate crimes, obtained from the SNCCP annotations; (3) all other police reports that are not annotated, and (4) difficult or complex cases according to experts at the SNCCP.

\subsubsection{The incitement against ethnic group dataset}

To study the performance of the estimators, we also created a subset of the dataset with known labels for all observations. This subset was created based on the crime of \textit{incitement against ethnic group} (``hets mot folkgrupp'' in Swedish), a hate crime by definition \cite[Ch 16, § 8]{SCC}. We use this to create a dataset with known labels for all observations, eliminating the need for manual annotation of police reports and enabling a Monte Carlo simulation. This is achieved by training a classifier to predict whether a police report constitutes \textit{incitement against an ethnic group} or not, and then generating predictions for a population of unseen cases. To maintain a class distribution similar to the overall hate crime classification problem, we sub-sampled the non-hate crime population to achieve a comparable balance.

\subsection{Step 1: Transformer neural network for hate crime classification} \label{sec:model_train}

We use pre-trained transformer encoder models BERT and roBERTa \cite{devlin2018bert,liu2019roberta} to develop a hate crime classifier based on the text of the police reports. Specifically, we use Swedish versions of these models, trained by the National Library of Sweden \cite{malmsten2020playing}. To adapt them to the specific classification task, we fine-tune them on previously annotated police reports, resulting in a hate-crime text classifier. In addition to fine-tuning the off-the-shelf models directly, we also study an intermediate domain adaptation (DA) step \cite{gururangan2020don}, to adapt the model to the specific language domain of police reports.

\subsubsection{Training, testing and validation}

As training data, we include all \num{49351} true hate crimes, \num{58881} true non-hate crimes, and \num{50000} other crimes, selected randomly (Table \ref{tab:train_val_test}). The true hate crimes are labeled $y_i=1$, while the true non-hate crimes and all other crimes are labelled $y_i=0$, assuming that randomly sampled police reports are non-hate crimes. Assuming that reports not flagged by the police are non-hate crimes will introduce a small amount of label noise, as some unflagged reports may be false negatives. However, we do not expect this to have a large impact due to the generally small proportion of hate crimes in the population \cite[see][for a discussion]{PUlearning}. 

The motivation behind including both true non-hate crimes and other crimes as $y=0$ is the distributional similarity between true hate crimes and true non-hate crimes. This similarity is a consequence of the keyword search/police flagging to extract reports for annotation. Hence, true hate crimes and confirmed non-hate crimes often share keywords. The randomly drawn reports assumed to be non-hate crimes contain signals to distinguish true hate crimes from a broad range of other crimes.

The validation split is a random sample of 15\% of the training data, comprising \num{23734} observations, with the same class balance as the training split. All observations from 2022 were held out as the test set and used for estimating the generalisation performance.

\begin{table}[]
    \centering
    \begin{tabular}{lcr}
       \toprule
       Training Data & Label & n  \\
       \midrule
       hate crime  & 1 & \num{49351} \\
       not hate crime  & 0 & \num{58881} \\
       other crimes  & 0 & \num{50000} \\
       \midrule
       Test Set &  &  \\
       \midrule
       hate crime & 1 & \num{2695} \\
       not hate crime & 0 & \num{1860} \\
       potential hate crime (complex case) & 0 & \num{233} \\
       unknown & NA & \num{1458974} \\
        \bottomrule
    \end{tabular}
    \caption{Number of police reports per category in training data and test set. The test set is all police reports from 2022. Here ``hate crimes'' and ``non-hate crimes'' are cases annotated by SNCCP, and ``other crimes'' are randomly selected from the pool of police reports that are not annotated. In addition, ``potential hate crimes'' is a possible hate crime, but the text lacks sufficient information to be certain. ``unknown'' is a non-annotated report.}
    \label{tab:train_val_test}
\end{table}

When fine-tuning the models, we used a maximum learning rate of 2e-5, a batch size of 16, the AdamW optimiser \cite{adamw}, and a linear learning rate scheduler with zero warm-up steps \cite{adamw_loshchilov2017decoupled,bengio2017deep}. We also employed early stopping based on the validation loss, as results from \cite{chalkidis2020legal} suggest that the recommended initial 2-4 epochs \cite{devlin2018bert} are not always sufficient for fine-tuning.

The \textit{incitement against ethnic group}-classifier was trained with the domain-adapted (see next section) BERT model on data from 2007-2021. All \num{15000} \textit{incitement against ethnic group} police reports were used as the positive class, together with a random sample of \num{30000} other crimes as the negative class. The held-out test set was 2022-data, comprising all 944 \textit{incitement against an ethnic group} cases and a random sample of \num{190000} other crimes. This random sample was selected to ensure a similar class balance to the hate crime classifier, and the same hyperparameter setting was used during fine-tuning.

\subsubsection{Domain adaptation}
Domain adaptation, i.e., continued pre-training of an encoder-based transformer model on a domain-specific corpus, has been shown to improve performance on the downstream task when using transformer encoder models \cite{vakili2022downstream, gururangan2020don}. This second pre-training phase was done by training the off-the-shelf roBERTa-model on the standard pre-training task of masked language modelling (MLM). Following results from \cite{wettig2022should}, however, we increased the masking probability to 20\% in contrast to the otherwise commonly used 15\% employed by many of BERT's predecessors \cite{liu2019roberta,joshi2020spanbert, lan2019albert, levine2020pmi, izsak2021train}. The data used for this pre-training was all police reports from 2007-2022.

We also used dynamic masking, meaning that the random masking of the training data is re-made for every epoch, as done in the original RoBERTa training \cite{liu2019roberta} while keeping the masking of the validation set fixed. A linear learning-rate scheduler with a 1\% warm-up period followed by a decay to zero was used \cite{gpt3,touvron2023llama,liu2019roberta}. Again, the maximum learning rate was 2e-5. We used the AdamW optimiser, a sequence length of 512 tokens, and a batch size of 16. The domain-adaptation training was run for ten epochs.

\subsection{Step 2 and 3: Finite population estimation} \label{sec:finite_pop_estim}

Our goal is to produce efficient estimates by leveraging predictions made by the encoder-based models as auxiliary information. On the test set, which includes all police reports from 2022, we derive estimates of the population total $t_p$, i.e. the total number of reported hate crimes in Sweden 2022, and police under-reporting $t_u=t_p-t_f$, where $t_f$ is the confirmed hate crimes in the flagged cases, using the estimators in Section \ref{sec:estimators} and the auxiliary variable from the hate crime classifier.

\section{Results}
We begin by evaluating the performance of the trained classifier and comparing it to the police annotations currently used by SNCCP. We then use the trained classifier on the known-total subset of the corpus to study the empirical properties of the estimator on real data and finalise by estimating true hate crimes for the year 2022.

\subsection{Step 1: Transformer classifier performance for Swedish hate crimes}

After training the transformer model on data from 2007 to 2021, we evaluated the model by predicting the expert-annotated cases in 2022, which consisted of 2,695 hate crimes and 1,860 non-hate crimes among the 5,788 reports flagged by the police as hate crimes. The performance of the models is shown in Table \ref{tab:model-results}. For the BERT-type models, the accuracies and F1-scores were approximately 95\% and 96\%, respectively. The roBERTa-type models showed slightly lower performance. Since the BERT-DA model had the highest F1 score, it was used for the preceding comparison with the police and the prediction-powered inference. 

Of the 233 complicated cases labelled ``Probably not'', this model predicted 144 as not hate crimes and 89 as hate crimes, showing the difficulty in these cases.

\begin{table}[ht]
  \centering
  \begin{tabular}{llllllrr}
    \toprule
    {Model}   & {Acc.} & {F1}    \\ 
    \midrule
    BERT     & 0.951  & 0.9580  \\
    BERT-DA & 0.951 & 0.9584  \\
    roBERTa   &  0.923  &  0.932  \\ 
    roBERTa-DA & 0.937  & 0.945  \\
    roBERTa-DA-2* & 0.923 & 0.932 \\
    \bottomrule
  \end{tabular}
  \caption{Performance of the models on a held-out test set consisting of expert-annotated cases from 2022 for standard and domain-adapted (DA) models.*Earlier version fine-tuned on training data with all true hate crimes, \num{25000} true non-hate crimes and \num{25000} other crimes, used for error analysis in \cite{sargeant2025}.}
  \label{tab:model-results}
\end{table}

Table \ref{tab:model-results} shows that the overall results are relatively good for classifying hate crimes. However, there is no larger difference between the models used here. 

For the roBERTa-DA-2 model, we did a more extensive analysis of the errors made by the model \cite[see][for details]{sargeant2025}. It is clear that the model struggles to correctly classify complex hate crimes, such as Quran burnings (due to legal complexity) and anti-Sámi hate crimes, due to small sample sizes.

\subsubsection{Classifier F1-score on the full 2022 test set}


Since only the flagged cases have known labels, we estimated the overall F1-score for all cases in 2022 using the BERT-DA model through a sampling procedure followed by the delta method (see Appendix~B for details).

To do this, we sampled and manually annotated cases in two strata. The first stratum consisted of cases where the police predicted \textit{not hate crime} but the model predicted \textit{hate crime} ($N_1 = 4964$). From this stratum, we drew an SRS of size $n = 200$, in which 104 cases were categorized  as hate crimes in the manual annotation. The second stratum consisted of cases where both the police and the model predicted \textit{not hate crime} ($N_0 =$\num{1458798}). Here, we drew a sample of size $n = 200$ with replacement and inclusion probabilities proportional to the model's predicted probabilities $\hat{p}_i$. No hate crimes were found in this sample.

Based on this stratified sample, we estimated the model’s F1-score to be $0.80 \pm 0.02$ on the 2022 test set.

The \textit{incitement against ethnic group} classifier had an F1-score of 68\% and accuracy of 99.6\% on the test set. The test set for this classifier contains 0.5\% hate crimes, i.e., it is heavily imbalanced, and even more so than the 2022 full test set for hate crimes. However, the results in F1 are similar in size.

\subsubsection{Comparing police annotation and the BERT-DA in 2022} \label{sec:f1_estim_methods}

As SNCCP currently operate, it use police flags, in a sense, as predictors for hate crime. Therefore, we evaluate the police performance and compare it with that of the BERT-DA classifier. To compare the classifier and the police, we focus on the cases where the two methods disagree. We use the manually annotated subset where the police predicted \textit{not hate crime} and the model predicted \textit{hate crime} ($N_1=4964$, stratum 1) described in the previous section, i.e., taking a sample of size $n_1=200$ out of which 99 samples were indeed a hate crime, leading to an estimated proportion of 0.50 (SE 0.035).

Similarly, we can study the stratum where the police indicate \textit{hate crime} and the model concludes \textit{not hate crime}. For this stratum, we already have the ground truth based on expert annotations from the SNCCP. Of the $N_2=1912$ observations in stratum 2, only 137 were actual hate crimes. Hence, out of the total number of differences between the model and the police, the BERT-DA model was correct in roughly 0.62 $[CI:0.55,0.69]$ in 2022. The resulting estimates in the two strata and the total number of errors can be found in Table \ref{tab:police_vs_model}. This indicates a higher accuracy of the model compared to the police annotation. 

\begin{table}[b]
    \centering
    \begin{tabular}{crrrr}
    \toprule
       Stratum & $N$ & $n$ & $\hat{p}$ & $\hat{t}$ \\
       \midrule
     1  & 4964 & 200 & 0.50 (0.035) & 2457 (172)\\
      2 & 1912 & - & 0.93 & 1775 \\
      \midrule
      Total & 6876 & - & 0.62 (0.035) & 4232 (172) \\      
      \bottomrule
    \end{tabular}
    \caption{Comparison in the proportion of correct classifications between the police and the transformer when models and police differ, where $\hat{p}$ is the proportion and $\hat{t}$ is the total of correct classifications for the transformer model, SE in parentheses. Stata 1 consists of observations where police indicate no hate crime while BERT-DA indicates hate crime. Note that in stratum 2, police indicate hate crime, BERT-DA does not, we know the true total.}
    \label{tab:police_vs_model}
\end{table}

\subsection{Step 2 and 3: Prediction-powered inference of Swedish hate crimes}

The transformer models for classification of hate crimes and \textit{incitement against ethnic group} trained in Section \ref{sec:model_train} are used for estimation using the approaches in Section \ref{sec:estimators}. With the \textit{incitement against ethnic group} classifier, we perform simulations estimating the total. With the hate crime classifier, we sample once and annotate manually.

\subsubsection{Monte Carlo experiments on known-population subset}

Based on the known-population subset of \textit{incitement against ethnic group} and with a sample size of $n=500$, the sampling distribution for the H2P2 estimator and the SbP estimation with SRS, can be seen in Figure \ref{fig:simulations}. We can see that the H2P2 estimator works well in terms of smaller variance, even for estimating a small proportion with $F_1$ of 0.68, and gives an unbiased estimate of the true total \cite{cochran1977sampling}. Based on the Monte Carlo simulations, we observe that the H2P2 estimator has a standard error of 102 (Table \ref{tab:monte_carlo}). To get a similar standard error for a random sample, we would need approximately \num{12900} observations. Hence, the classifier improved the sampling efficiency quite remarkable. It can, however, be noted that a slight drawback is that the estimators sampling distribution has a thick right tail. This occurs when false negatives with very small inclusion probabilities are included in the sample.

\begin{figure}[t]
    \centering
    \includegraphics[width=0.9\textwidth]{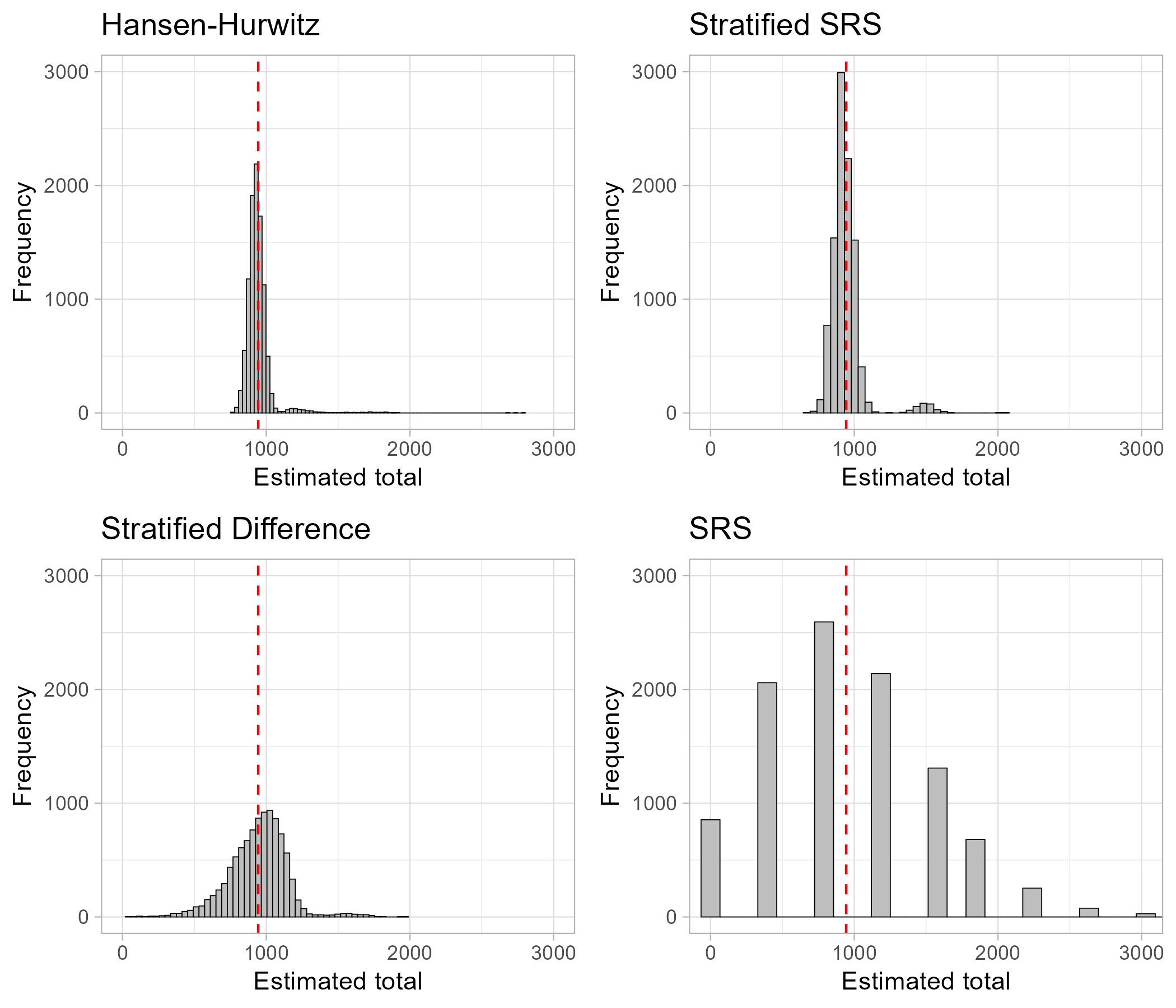}
    \caption{Histogram of simulating \num{10000} H2P2 estimates, SRS without stratification, SbP using SRS, SbP with difference estimation in the zero-stratum within the subset of \textit{incitement against ethnic group} with $n=500$. The red line is the true total.}
    \label{fig:simulations}
\end{figure}

For stratification by prediction, Neyman allocation was applied based on the within-stratum population variances, with simple random sampling (SRS) used within each stratum. The resulting sample sizes were 231 in the one-stratum and 269 in the zero-stratum. In contrast to the H2P2 estimator, we can observe the issue with the zero-stratum. Many of the samples do not contain any $y_i=1$. The estimated total in that stratum using the standard SRS variance estimator is then zero. However, when at least one $y=1$ is drawn present in that stratum, the estimated total is substantially larger, resulting in a bimodal sampling distribution. This result emphasises the problem of having a binary target variable in an highly imbalanced population and using a stratified SRS.

One way to solve this is by using the difference estimator within the zero-stratum. If the stratum consisting of predicted non-hate crimes contain no true hate crimes, the estimated total in that stratum is $\hat{t}_{D}=t_{\hat{p}}-\hat{t}_{\hat{p}}$. This is the difference between the estimated total of the auxiliary variable and the true total of the auxiliary variable, which is, in practice, always nonzero. However, the drawback of this estimator is the larger variance. As a baseline comparison, SRS with no stratification is also included, showing clear signs of the the rare number of \textit{incitement against ethnic group}. The large variance can be seen in Figure \ref{fig:simulations}.

\begin{table}[b]
    \centering
    \begin{tabular}{lrr}
    \toprule
       Estimator  & SE & DEFF \\
       \midrule
       SRS & 600 & \\
       HH & 102 & 0.029 \\
       Stratified SRS & 116 & 0.038 \\
       Stratified Diff & 203 & 0.11 \\
       \bottomrule
    \end{tabular}
    \caption{Standard errors (SE) and design effect (DEFF) for Monte Carlo estimation of population total in known subset of \textit{incitement against ethnic group}. The estimators are HH, stratified SRS, and difference estimator. We used 10~000 simulations with a sample size of $n=500$ and Neyman allocation when stratifying. The difference estimator was only used in the zero-stratum. In the one-stratum, SRS was used.}
    \label{tab:monte_carlo}
\end{table}

We can summarise the results as the H2P2 estimator is a suitable choice if statistical efficiency is important. However, if the setting is more like that of SNCCP, in which the main use of a classifier would be for stratification purposes, combining a stratification-by-prediction and a difference estimator is more practical.

\subsubsection{Estimating the total Swedish hate crimes in 2022 and police under-reporting} \label{sec:total_estim}

After obtaining the auxiliary variable with the transformer classifier, we estimate the population total in the 2022 police reports by sampling and manually annotating. The annotation here was done by one of the authors, and hence the quality is not the same as SNCCP experts. Based on using the H2P2 estimator and BERT-DA, with a sample size of $n=200$, the total number of hate crimes was estimated to be \num{6051}, with a standard error of 548 (see Appendix A for the distribution of the auxiliary variable $\hat{p_i}$ from the BERT-DA model in this sample).

This result can be compared to a situation where we have no auxiliary information and take an SRS. In such a case, the same standard error as the H2P2 estimator is reached when taking a sample of approximately \num{28000} observations. Compared to the sample size of $n=200$ used with the H2P2 here. This implies a design effect of 0.0068, highlighting again the usefulness of prediction-powered estimation.

In the SbP, we used equal allocation with ($n=100)$ in the one-stratum and ($n=100$) in the zero-stratum. 
The results after annotating were 60 hate crimes in the one-stratum and zero hate crimes in the zero-stratum, giving an estimated total of 4618 hate crimes. Since the variance estimate in the zero-stratum is zero, the variance estimate for the total is likely an underestimate, again highlighting the issue of using SRS in the zero-stratum. When applying the difference estimator in the zero-stratum to mitigate this, the estimated total was 6193 with an estimated standard error of 1220.

Finally, we estimate the number of hate crime missed in 2022 by subtracting the confirmed number of hate crimes in the flagged cases 2022 from the H2P2 estimate. The result is an estimated under-reporting of between 2260 and 4452 hate crimes compared to the estimate based purely on police flagged hate crimes.

\begin{table}[h]
    \label{surveytable}
    \centering
    \begin{tabular}{lrr}
    \toprule
        Estimator  & Estimate & SE  \\
        \midrule
        HH  & 6051 & 548  \\
        Stratified SRS  & 4618 & 263  \\
        Stratified Diff & 6193 & 1220 \\
    \bottomrule
    \end{tabular}
    \caption{Estimated total and standard error of hate crimes in police reports 2022, using HH, SBP with SRS and difference estimation (only in the zero-stratum). The BERT-DA model was used to create the predictions. $N=$\num{1463762}, $n=200$, with $n_1=100$ and $n_2=100$ when using SBP.}
\end{table}

\newpage
\section{Discussion and conclusion}

We introduce a simple yet effective approach to use model predictions as an informative auxiliary variable in sampling estimators to obtain unbiased finite population estimates. We demonstrate the applicability in a real-world setting of Swedish hate crime estimation. The predictive model performed well on the labelled section of the test set, with both accuracy and F1-score above 95\%. The overall F1-score, obtained by estimating the model's performance also on non-annotated reports, was 80\%. 
Based on the comparison of the trained transformer and the police, we found that the model outperforms the police in classifying and annotating hate crimes, opening up new possibilities for prediction-powered estimation of hate crime statistics and other similar settings.

We then examined how the predictions from this model could be utilised as auxiliary information in three distinct ways when estimating the total number of hate crimes in Swedish police reports. Incorporating the model predictions into the estimation can greatly improve efficiency, with an approximate design effect of 0.0068. Using the H2P2 estimator with a sample of size 200 gave an estimate of 6051 $[CI: 4955, 7147]$ hate crimes in the police reports 2022. From this, we estimated that the Swedish police missed to flag somewhere between 2,260 and 4,452 true hate crimes in their annotations in 2022.

However, in practice, domain estimation is also of interest in the production of official statistics. In the case of hate crimes, this can be the hate crime type, the type of environment where the offence took place, or the relational status of the victim and perpetrator. When estimating the total for the entire population using a stratified random sampling approach, a limitation arises from estimation in the zero-stratum. The reason is the large number of police reports in this stratum, in combination with the low proportion of actual hate crimes — especially if the classifier is accurate, resulting in a bimodal sampling distribution. The difference estimator, or the H2P2 estimator, can estimate the proportion in the zero-stratum to mitigate the problem.

There are multiple directions of future work, both regarding estimation and the development of classifiers for statistical purposes. One is to further study the theoretical properties of prediction-powered sampling estimators. This would also include extending the results in this paper to the use of the Horwitz-Thompson estimator instead of the Hansen-Hurwitz estimator. 

As shown, the limitation of the method is the need for a good classifier. Without high-performance classifiers, there will be no performance gains for the estimator. Hence, how to construct well-functioning prediction-powered estimation in a statistical setting is an important field for future studies.

The use of an LLM is a promising approach, enabling prediction-powered finite population estimation with significantly less training data. It may also improve challenging contextual textual classification tasks, as identified in \cite{sargeant2025}, for example, by facilitating the identification of specific hate crime categories. However, to use the H2P2 estimator, a numerical auxiliary variable is needed, and LLMs does not provide this by default. However, the sensitive data used in this paper implies that the models need to be stored and used locally. With the current hardware available at SNCCP, only very small LLMs could be used.

The primary motivational use case for this paper is the analysis of public hate crime statistics. However, this approach should work well in many similar situations where large corpora of textual documents are of interest for the production of statistics. In particular, this could include other areas of crime or legal statistics, such as fraud or domestic violence. 

The methodology presented here offers a robust framework for transforming unstructured textual data into meaningful statistical insights that can inform evidence-based policy decisions across multiple domains of social concern.

\subsubsection*{Ethics Statement}
The ethical review was carried out by the Swedish Ethical Review Authority in accordance with the Act (2003:460) concerning the ethical review of research involving humans. The Swedish Ethical Review Authority is an independent governmental agency responsible for evaluating research proposals to ensure compliance with Swedish laws and ethical standards. For this review, we submitted an application consisting of a general research plan, the principal investigator's CV, and a standardized form addressing research objectives, methods, ethical considerations, and timeline. These documents were subsequently assessed by the Swedish Ethical Review Authority. 
The ethical review was assigned the number 2022-05586-01 and was approved on 25 October 2022.

\bibliographystyle{unsrt} 
\bibliography{refs}

\begin{thebibliography}{10}

\bibitem{cochran1977sampling}
William~Gemmell Cochran.
\newblock {\em {Sampling techniques}}.
\newblock John Wiley \& Sons, 3 edition, 1977.

\bibitem{hansen1943theory}
Morris~H Hansen and William~N Hurwitz.
\newblock {On the theory of sampling from finite populations}.
\newblock {\em The Annals of Mathematical Statistics}, 14(4):333--362, 1943.

\bibitem{horvitz1952}
Daniel~G Horvitz and Donovan~J Thompson.
\newblock {A generalization of sampling without replacement from a finite universe}.
\newblock {\em Journal of the American Statistical Association}, 47(260):663--685, 1952.

\bibitem{sarndal2003model}
Carl-Erik S{\"a}rndal, Bengt Swensson, and Jan Wretman.
\newblock {\em {Model assisted survey sampling}}.
\newblock Springer New York, NY, 1 edition, 2003.

\bibitem{magnussen2011horvitz}
Steen Magnussen.
\newblock {A Horvitz--Thompson-type estimator of species richness}.
\newblock {\em Environmetrics}, 22(7):901--910, 2011.

\bibitem{rudolph2014estimating}
Kara~E Rudolph, Iv{\'a}n D{\'\i}az, Michael Rosenblum, and Elizabeth~A Stuart.
\newblock {Estimating population treatment effects from a survey subsample}.
\newblock {\em American Journal of Epidemiology}, 180(7):737--748, 2014.

\bibitem{boo2022high}
Gianluca Boo, Edith Darin, Douglas~R Leasure, Claire~A Dooley, Heather~R Chamberlain, Attila~N L{\'a}z{\'a}r, Kevin Tschirhart, Cyrus Sinai, Nicole~A Hoff, Trevon Fuller, et~al.
\newblock {High-resolution population estimation using household survey data and building footprints}.
\newblock {\em Nature communications}, 13(1):1330, 2022.

\bibitem{tille2022some}
Yves Till{\'e}, Marc Debusschere, Henri Luomaranta, Martin Axelson, Eva Elvers, Anders Holmberg, and Richard Valliant.
\newblock {Some Thoughts on Official Statistics and its Future (with discussion)}.
\newblock {\em Journal of Official Statistics}, 38(2):557--598, 2022.

\bibitem{StatisticsCanada2017}
{Statistics Canada}.
\newblock {Statistics Canada’s Quality Assurance Framework}.
\newblock \url{https://www150.statcan.gc.ca/n1/en/catalogue/12-586-X}, 2017.
\newblock Statistics Canada, Ottawa.

\bibitem{swedish_law}
{Riksdagen [The Swedish Parliament]}.
\newblock {Lag om den officiella statistiken [Official Statistics Act]}.
\newblock \url{https://www.riksdagen.se/sv/dokument-och-lagar/dokument/svensk-forfattningssamling/lag-200199-om-den-officiell_sfs-2001-99/}, 2001.
\newblock SFS 2001:99.

\bibitem{EUReg223_2009}
{European Parliament and Council}.
\newblock {Regulation (EC) No 223/2009 of the European Parliament and of the Council of 11 March 2009 on European Statistics}, 2009.
\newblock OJ L 87, p 164–173.

\bibitem{NAP27934}
{National Academies of Sciences, Engineering, and Medicine}.
\newblock {\em Principles and Practices for a Federal Statistical Agency: Eighth Edition}.
\newblock The National Academies Press, Washington, DC, 2024.

\bibitem{bonikowski2022}
Bart Bonikowski, Yuchen Luo, and Oscar Stuhler.
\newblock {Politics as Usual? Measuring Populism, Nationalism, and Authoritarianism in US Presidential Campaigns (1952--2020) with Neural Language Models}.
\newblock {\em Sociological Methods \& Research}, 51(4):1721--1787, 2022.

\bibitem{do_ollion2024augmented}
Salom{\'e} Do, {\'E}tienne Ollion, and Rubing Shen.
\newblock {The augmented social scientist: Using sequential transfer learning to annotate millions of texts with human-level accuracy}.
\newblock {\em Sociological Methods \& Research}, 53(3):1167--1200, 2024.

\bibitem{luo2021deep}
Xiao Luo, Priyanka Gandhi, Susan Storey, and Kun Huang.
\newblock {A deep language model for symptom extraction from clinical text and its application to extract COVID-19 symptoms from social media}.
\newblock {\em IEEE Journal of Biomedical and Health Informatics}, 26(4):1737--1748, 2021.

\bibitem{yang2025emerging}
Baocheng Yang, Bing Zhang, Kevin Cutsforth, Shanfu Yu, and Xiaowen Yu.
\newblock {Emerging industry classification based on BERT model}.
\newblock {\em Information Systems}, 128:102484, 2025.

\bibitem{rilkoff2024innovations}
Heather Rilkoff, Shannon Struck, Chelsea Ziegler, Laura Faye, Dana Paquette, and David Buckeridge.
\newblock {Innovations in public health surveillance: An overview of novel use of data and analytic methods}.
\newblock {\em Canada Communicable Disease Report}, 50(3-4):93, 2024.

\bibitem{anmalda_brott}
{Brottsförebyggande rådet [Swedish National Council for Crime Prevention] }.
\newblock {Anmälda brott 2023 slutlig statistik [Reported Crimes 2023 Final Statistics]}.
\newblock \url{https://bra.se/download/18.3d22500318e5f70da2846d9/1711454523087/Statistikrapport_anmalda_2023.pdf}, 2024.
\newblock Accessed: 2024-08-25.

\bibitem{settles2009active}
Burr Settles.
\newblock {Active Learning Literature Survey}.
\newblock Computer Sciences Technical Report 1648, University of Wisconsin--Madison, 2009.

\bibitem{LundgrenLejonstad2023}
Jon Lundgren and Aravella Lejonstad.
\newblock {Polisanmälda hatbrott 2022 - En sammanställning av de ärenden som hatbrottsmarkerats av polisen [Hate crimes in police reports 2022 - a summary of reports marked as hate crimes by the police]}.
\newblock {Brottsförebyggande rådet [Swedish National Council for Crime Prevention] Report nr 2023:16}, 2023.

\bibitem{angelopoulos2023prediction}
Anastasios~N Angelopoulos, Stephen Bates, Clara Fannjiang, Michael~I Jordan, and Tijana Zrnic.
\newblock {Prediction-powered inference}.
\newblock {\em Science}, 382(6671):669--674, 2023.

\bibitem{egami}
Naoki Egami, Musashi Hinck, Brandon Stewart, and Hanying Wei.
\newblock {Using imperfect surrogates for downstream inference: Design-based supervised learning for social science applications of large language models}.
\newblock {\em Advances in Neural Information Processing Systems}, 36, 2024.

\bibitem{eisenstein2018natural}
Jacob Eisenstein.
\newblock {\em {Natural language processing}}.
\newblock MIT Press, 2019.

\bibitem{wankhade2022survey}
Mayur Wankhade, Annavarapu Chandra~Sekhara Rao, and Chaitanya Kulkarni.
\newblock A survey on sentiment analysis methods, applications, and challenges.
\newblock {\em Artificial Intelligence Review}, 55(7):5731--5780, 2022.

\bibitem{piryani2017analytical}
Rajesh Piryani, Devaraj Madhavi, and Vivek~Kumar Singh.
\newblock Analytical mapping of opinion mining and sentiment analysis research during 2000--2015.
\newblock {\em Information Processing \& Management}, 53(1):122--150, 2017.

\bibitem{kowsari2019text}
Kamran Kowsari, Kiana Jafari~Meimandi, Mojtaba Heidarysafa, Sanjana Mendu, Laura Barnes, and Donald Brown.
\newblock Text classification algorithms: A survey.
\newblock {\em Information}, 10(4):150, 2019.

\bibitem{vaswani2017attention}
Ashish Vaswani, Noam Shazeer, Niki Parmar, Jakob Uszkoreit, Llion Jones, Aidan~N Gomez, {\L}ukasz Kaiser, and Illia Polosukhin.
\newblock {Attention is all you need}.
\newblock {\em Advances in neural information processing systems}, 30, 2017.

\bibitem{devlin2018bert}
Jacob Devlin, Ming-Wei Chang, Kenton Lee, and Kristina Toutanova.
\newblock {Bert: Pre-training of deep bidirectional transformers for language understanding}.
\newblock {\em arXiv preprint arXiv:1810.04805}, 2018.

\bibitem{liu2019roberta}
Yinhan Liu, Myle Ott, Naman Goyal, Jingfei Du, Mandar Joshi, Danqi Chen, Omer Levy, Mike Lewis, Luke Zettlemoyer, and Veselin Stoyanov.
\newblock {Roberta: A robustly optimized bert pretraining approach}.
\newblock {\em arXiv preprint arXiv:1907.11692}, 2019.

\bibitem{beltagy2019scibert}
Iz~Beltagy, Kyle Lo, and Arman Cohan.
\newblock {SciBERT: A pretrained language model for scientific text}.
\newblock {\em arXiv preprint arXiv:1903.10676}, 2019.

\bibitem{sun2019fine}
Chi Sun, Xipeng Qiu, Yige Xu, and Xuanjing Huang.
\newblock {How to fine-tune bert for text classification?}
\newblock In {\em Chinese computational linguistics: 18th China national conference, CCL 2019, Kunming, China, October 18--20, 2019, proceedings 18}, pages 194--206. Springer, 2019.

\bibitem{gasparetto2022survey}
Andrea Gasparetto, Matteo Marcuzzo, Alessandro Zangari, and Andrea Albarelli.
\newblock {A survey on text classification algorithms: From text to predictions}.
\newblock {\em Information}, 13(2):83, 2022.

\bibitem{ollion2023chatgpt}
Etienne Ollion, Rubing Shen, Ana Macanovic, and Arnault Chatelain.
\newblock {Chatgpt for Text Annotation? Mind the Hype!}
\newblock {\em SocArXiv preprint}, 2023.

\bibitem{fields2024survey}
John Fields, Kevin Chovanec, and Praveen Madiraju.
\newblock {A survey of text classification with transformers: How wide? how large? how long? how accurate? how expensive? how safe?}
\newblock {\em IEEE Access}, 12:6518--6531, 2024.

\bibitem{kristensen2025chatbots}
Ross~Deans Kristensen-McLachlan, Miceal Canavan, Marton K{\'a}rdos, Mia Jacobsen, and Lene Aar{\o}e.
\newblock Are chatbots reliable text annotators? sometimes.
\newblock {\em PNAS nexus}, 4(4):pgaf069, 2025.

\bibitem{samsi2023words}
Siddharth Samsi, Dan Zhao, Joseph McDonald, Baolin Li, Adam Michaleas, Michael Jones, William Bergeron, Jeremy Kepner, Devesh Tiwari, and Vijay Gadepally.
\newblock From words to watts: Benchmarking the energy costs of large language model inference.
\newblock In {\em 2023 IEEE High Performance Extreme Computing Conference (HPEC)}, pages 1--9. IEEE, 2023.

\bibitem{beck2018machine}
Martin Beck, Florian Dumpert, and Joerg Feuerhake.
\newblock {Machine learning in official statistics}.
\newblock {\em arXiv preprint arXiv:1812.10422}, 2018.

\bibitem{mcconville2017lasso}
Kelly~S McConville, F~Jay Breidt, Thomas~CM Lee, and Gretchen~G Moisen.
\newblock {Model-assisted survey regression estimation with the lasso}.
\newblock {\em Journal of Survey Statistics and Methodology}, 5(2):131--158, 2017.

\bibitem{SCC}
{Sveriges Riksdag}.
\newblock {Brottsbalken [Swedish Penal Code]}, 1962.
\newblock SFS 1962:700.

\bibitem{TODO}
To~Do.
\newblock {Reference we need to fix}.
\newblock Not at all, 2099.

\bibitem{snccp2018}
Carina Djärv and Anna Gavell~Frenzel.
\newblock {Polisens hatbrottsmarkering - Kvalitetsgranskning av polisens hatbrottsmarkering samt jämförelse mellan statistik baserad på polisens hatbrottsmarkering och Brås hatbrottsstatistik [Police hate crime flagging - Quality Control of police hate crime flagging and comparison between statistics based on police hate crime flaggings and SNCCP hate crime statistics]}.
\newblock Technical Report 2018:13, Brottsförebyggande rådet, 2018.

\bibitem{malmsten2020playing}
Martin Malmsten, Love B{\"o}rjeson, and Chris Haffenden.
\newblock {Playing with Words at the National Library of Sweden - Making a Swedish BERT}.
\newblock {\em arXiv preprint arXiv:2007.01658}, 2020.

\bibitem{gururangan2020don}
Suchin Gururangan, Ana Marasovi{\'c}, Swabha Swayamdipta, Kyle Lo, Iz~Beltagy, Doug Downey, and Noah~A Smith.
\newblock {Don't stop pretraining: Adapt language models to domains and tasks}.
\newblock {\em arXiv preprint arXiv:2004.10964}, 2020.

\bibitem{PUlearning}
Jessa Bekker and Jesse Davis.
\newblock {Learning from positive and unlabeled data: A survey}.
\newblock {\em Machine Learning}, 109(4):719--760, 2020.

\bibitem{adamw}
Ilya Loshchilov and Frank Hutter.
\newblock {Decoupled Weight Decay Regularization}.
\newblock In {\em International Conference on Learning Representations (ICLR)}, 2019.

\bibitem{adamw_loshchilov2017decoupled}
Ilya Loshchilov and Frank Hutter.
\newblock {Decoupled weight decay regularization}.
\newblock {\em arXiv preprint arXiv:1711.05101}, 2017.

\bibitem{bengio2017deep}
Ian Goodfellow, Yoshua Bengio, and Aaron Courville.
\newblock {\em {Deep Learning}}.
\newblock MIT Press, 2016.
\newblock \url{http://www.deeplearningbook.org}.

\bibitem{chalkidis2020legal}
Ilias Chalkidis, Manos Fergadiotis, Prodromos Malakasiotis, Nikolaos Aletras, and Ion Androutsopoulos.
\newblock {LEGAL-BERT: The muppets straight out of law school}.
\newblock {\em arXiv preprint arXiv:2010.02559}, 2020.

\bibitem{vakili2022downstream}
Thomas Vakili, Anastasios Lamproudis, Aron Henriksson, and Hercules Dalianis.
\newblock {Downstream task performance of BERT models pre-trained using automatically de-identified clinical data}.
\newblock In {\em Proceedings of the Thirteenth Language Resources and Evaluation Conference}, pages 4245--4252, 2022.

\bibitem{wettig2022should}
Alexander Wettig, Tianyu Gao, Zexuan Zhong, and Danqi Chen.
\newblock {Should you mask 15\% in masked language modeling?}
\newblock {\em arXiv preprint arXiv:2202.08005}, 2022.

\bibitem{joshi2020spanbert}
Mandar Joshi, Danqi Chen, Yinhan Liu, Daniel~S Weld, Luke Zettlemoyer, and Omer Levy.
\newblock {Spanbert: Improving pre-training by representing and predicting spans}.
\newblock {\em Transactions of the association for computational linguistics}, 8:64--77, 2020.

\bibitem{lan2019albert}
Zhenzhong Lan, Mingda Chen, Sebastian Goodman, Kevin Gimpel, Piyush Sharma, and Radu Soricut.
\newblock {Albert: A lite bert for self-supervised learning of language representations}.
\newblock {\em arXiv preprint arXiv:1909.11942}, 2019.

\bibitem{levine2020pmi}
Yoav Levine, Barak Lenz, Opher Lieber, Omri Abend, Kevin Leyton-Brown, Moshe Tennenholtz, and Yoav Shoham.
\newblock {Pmi-masking: Principled masking of correlated spans}.
\newblock {\em arXiv preprint arXiv:2010.01825}, 2020.

\bibitem{izsak2021train}
Peter Izsak, Moshe Berchansky, and Omer Levy.
\newblock {How to train BERT with an academic budget}.
\newblock {\em arXiv preprint arXiv:2104.07705}, 2021.

\bibitem{gpt3}
Tom Brown, Benjamin Mann, Nick Ryder, Melanie Subbiah, Jared~D Kaplan, Prafulla Dhariwal, Arvind Neelakantan, Pranav Shyam, Girish Sastry, Amanda Askell, et~al.
\newblock {Language models are few-shot learners}.
\newblock {\em Advances in neural information processing systems}, 33:1877--1901, 2020.

\bibitem{touvron2023llama}
Hugo Touvron, Thibaut Lavril, Gautier Izacard, Xavier Martinet, Marie-Anne Lachaux, Timoth{\'e}e Lacroix, Baptiste Rozi{\`e}re, Naman Goyal, Eric Hambro, Faisal Azhar, et~al.
\newblock {Llama: Open and efficient foundation language models}.
\newblock {\em arXiv preprint arXiv:2302.13971}, 2023.

\bibitem{sargeant2025}
Holli Sargeant, Hannes Waldetoft, and Måns Magnusson.
\newblock {Classifying Hate: Legal and Ethical Evaluations of ML-Assisted Hate Crime Classification and Estimation in Sweden}.
\newblock In {\em Proceedings of the 2025 ACM Conference on Fairness, Accountability, and Transparency}, FAccT '25, New York, USA, 2025. Association for Computing Machinery.
\newblock DOI: 10.1145/3715275.3732016.

\end{thebibliography}

\newpage
\appendix
\section*{Appendix}
\section{Auxiliary variable distribution}

Figure \ref{fig:HH_aux_hist} shows the distribution of the auxiliary variable, $\hat{p}$, on the observations in the sample used for the Hansen-Hurwits estimator, along with the distribution of $\hat{p}$ in all observations from 2022. It is clear that the sample increases the probability of a true hate crime being included in the sample, comparing to taking an SRS.

\begin{figure}[h]
    \centering
    \includegraphics[width=1\textwidth]{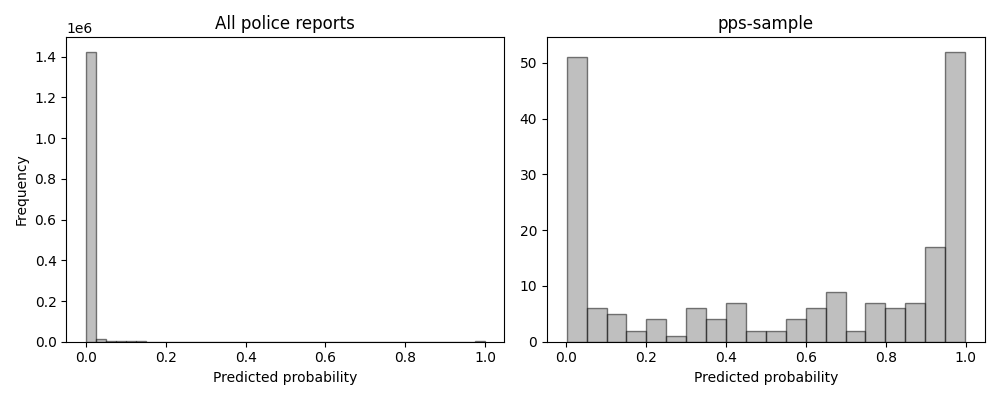}
    \caption{Histogram of the auxiliary variable, $\hat{p}$, in all police reports from 2022, and when in a sample of $n=200$ taken with replacement and inclusion probability proportional $\hat{p}$.}
    \label{fig:HH_aux_hist}
\end{figure}

\section{Delta method for transformer F1-score}

Since the F1-score is a differentiable function of the true positives (TP), false positives (FP) and false negatives (FN), we can use the delta method to find the derivative. The delta method is

$$
\nabla f(\theta)^T \cdot \frac{\Sigma}{n} \cdot \nabla f(\theta)
$$

We first notice that FP can be expressed in terms of TP as $C-TP$ where $C$ is a known constant. We then set $TP=\theta_1$ and $FN=\theta_2$, and get the expression

$$
f(TP,FP,FN) = \frac{2 \times TP}{TP+FN+C}
$$

Furthermore, in the stratum where the police and model both predicted ``not hate crime'', of $n=200$ samples, no hate crimes were present. Since the estimated FN is then zero, and hence its variance also, the problem further simplifies to a one-parameter problem of $TP$. The first-order derivative w.r.t $TP$ is 

$$
\frac{\partial f(TP,FN,FP)}{\partial TP} = \frac{2(FP+C)}{(FN+FP+C)^2}
$$

Finally, we get the estimate as 

$$
V(f(TP,FP,FN))=\left(\frac{\partial f(TP,FN,FP)}{\partial TP}\right)^2\hat{V}(\hat{TP})=2.6\times 10^{-4}, 
$$
where $\hat{V}(\hat{TP})$ is the estimated variance of TP. From this we get $SE_{F1}=\sqrt{\hat{V}(\hat{TP})}=0.02$.

\end{document}